\documentclass[journal]{IEEEtran}
\IEEEoverridecommandlockouts
\usepackage{setspace}
\usepackage{cite}
\usepackage{amsmath,amssymb,amsfonts,mathtools,bm}
\usepackage{amsthm}
\usepackage{algorithm}
\usepackage{graphicx}
\usepackage{textcomp}
\usepackage{xcolor,float}
\usepackage{tabularx}
\usepackage{textcomp}
\usepackage{diagbox}
\usepackage{svg}
\usepackage{booktabs}
\usepackage{subfigure}
\usepackage{hyperref}
\newtheorem{theorem}{Theorem}
\newtheorem{problem}{Problem}
\usepackage{algpseudocode}
\usepackage{graphicx}
\usepackage{makecell}
\usepackage{amsthm}
\usepackage{caption}
\usepackage{sidecap}
\usepackage{microtype}
\usepackage{booktabs} 
\usepackage{multirow}
\usepackage{float}
\usepackage{adjustbox} 
\usepackage{amsmath}
\usepackage{amssymb}
\usepackage{colortbl}
\usepackage{makecell}
\usepackage{float}
\usepackage{url}
\usepackage{balance}
\usepackage{bbding}
\usepackage{hyperref}
\usepackage{graphicx}
\usepackage{sidecap}
\usepackage{microtype}
\usepackage{graphicx}
\usepackage{subfigure}
\usepackage{booktabs} 
\usepackage{multirow}
\usepackage{array}
\usepackage{subcaption}
\usepackage{graphicx}
\usepackage{makecell}
\usepackage{amsmath}
\usepackage{amssymb}

\usepackage{tikz}
\usepackage[final]{changes}
\usepackage{ifthen}
\newboolean{showchanges}
\setboolean{showchanges}{true}  

\newcommand{\add}[1]{%
    \ifthenelse{\boolean{showchanges}}%
        {\textcolor{blue}{#1}}
        {#1\relax}
}

\definecolor{lime}{HTML}{A6CE39}
\DeclareRobustCommand{\orcidicon}{%
    \begin{tikzpicture}
    \draw[lime, fill=lime] (0,0) 
    circle [radius=0.16] 
    node[white] {{\fontfamily{qag}\selectfont \tiny ID}};    \draw[white, fill=white] (-0.0625,0.095) 
    circle [radius=0.007];    \end{tikzpicture}
    \hspace{-2mm}}
\foreach \x in {A, ..., Z}{%
    \expandafter\xdef\csname orcid\x\endcsname{\noexpand\href{https://orcid.org/\csname orcidauthor\x\endcsname}{\noexpand\orcidicon}}
    }

\allowdisplaybreaks
\renewcommand{\baselinestretch}{1}

\begin{document}

\title{RadioDiff-Flux: Efficient Radio Map Construction via Generative Denoise Diffusion Model Trajectory Midpoint Reuse}
\author{
Xiucheng Wang,\orcidA{} ~\IEEEmembership{Graduate Student Member,~IEEE,}
Peilin Zheng, \orcidB{} ~\IEEEmembership{Graduate Student Member,~IEEE,}
Honggang Jia, \orcidC{} ~\IEEEmembership{Graduate Student Member,~IEEE,}
Nan Cheng,\orcidD{}~\IEEEmembership{Senior Member,~IEEE,}
Ruijin Sun,\orcidE{}~\IEEEmembership{Member,~IEEE,}\\
Conghao Zhou,\orcidF{}~\IEEEmembership{Member,~IEEE,}
Xuemin (Sherman) Shen,\orcidG{}~\IEEEmembership{Fellow,~IEEE}

\thanks{ }
\thanks{
\par This work was supported by the National Key Research and Development Program of China (2024YFB907500).
\par Xiucheng Wang, Peilin Zheng, Honggang Jia, Nan Cheng, and Ruijin Sun are with the State Key Laboratory of ISN and School of Telecommunications Engineering, Xidian University, Xi’an 710071, China. (e-mail: \{xcwang\_1, plzheng, jiahg\}@stu.xidian.edu.cn; dr.nan.cheng@ieee.org; sunruijin,@xidian.edu.cn; conghao.zhou@ieee.org). (Peilin Zheng and Honggang Jia contribute equally.) \textit{Nan Cheng is the corresponding author}.
\par Xuemin (Sherman) Shen is with the Department of Electrical and Computer Engineering, University of Waterloo, Waterloo, N2L 3G1, Canada (e-mail: sshen@uwaterloo.ca).
}

}

    \maketitle

\IEEEdisplaynontitleabstractindextext

\IEEEpeerreviewmaketitle

\begin{abstract}
Accurate radio map (RM) construction is essential to enabling environment-aware and adaptive wireless communication. However, in future 6G scenarios characterized by high-speed network entities and fast-changing environments, it is very challenging to meet real-time requirements. Although generative diffusion models (DMs) can achieve state-of-the-art accuracy with second-level delay, their iterative nature leads to prohibitive inference latency in delay-sensitive scenarios. In this paper, by uncovering a key structural property of diffusion processes: the latent midpoints remain highly consistent across semantically similar scenes, we propose RadioDiff-Flux, a novel two-stage latent diffusion framework that decouples static environmental modeling from dynamic refinement, enabling the reuse of precomputed midpoints to bypass redundant denoising. In particular, the first stage generates a coarse latent representation using only static scene features, which can be cached and shared across similar scenarios. The second stage adapts this representation to dynamic conditions and transmitter locations using a pre-trained model, thereby avoiding repeated early-stage computation. The proposed RadioDiff-Flux significantly reduces inference time while preserving fidelity. Experiment results show that RadioDiff-Flux can achieve up to 50× acceleration with less than 0.15\% accuracy loss, demonstrating its practical utility for fast, scalable RM generation in future 6G networks.
\end{abstract}


\section{Introduction}
\begin{figure}[t]
    \centering
    \includegraphics[width=0.95\linewidth]{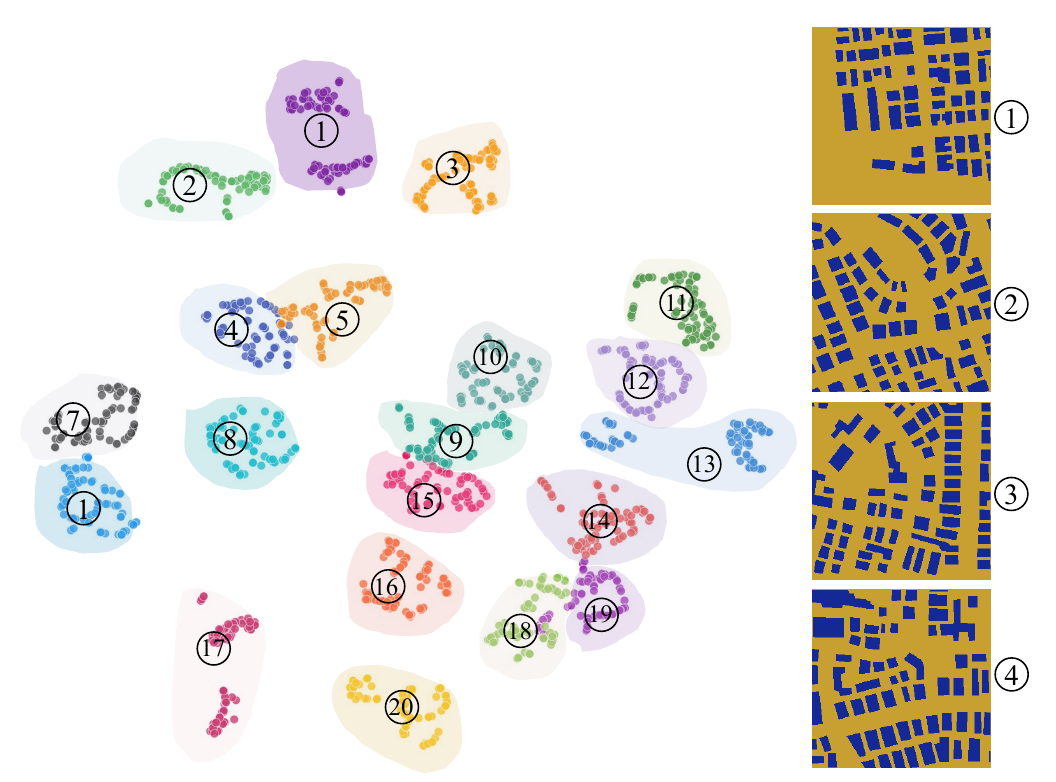}
    \caption{The illustration of the similarity of latent variables for RMs.}\label{fig-cluster}
    \vspace{-15pt}
\end{figure}
The increasing demand for efficient and adaptive channel estimation methods in 6G networks has shifted the focus from traditional pilot signal-based measurements to computational approaches \cite{han2020channel,ma2022satellite,6g}. This is primarily driven by the need to estimate channel characteristics in large-dimensional environments, which are common in 6G, as well as the integration of passive devices such as Intelligent Reflective Surfaces (IRS) \cite{wang2022massive,Jiang2021a,ozdogan2019intelligent}. Additionally, the pre-planning of movement paths for mobile wireless access nodes, such as drones and satellites, introduces further complexity in channel estimation, as these nodes must account for their dynamic positions before reaching target areas \cite{cheng2019space,wang2022joint}. In response to these challenges, Radio Maps (RMs) \cite{levie2021radiounet} and Channel Knowledge Maps (CKMs) \cite{zeng2024tutorial} have emerged as important tools for visually representing the spatial distribution of wireless channel features via pre-computation. Although these methods are effective in capturing the accuracy of spatial distributions, they often fail to address the growing need for efficient construction and real-time adaptability, especially when environmental factors or wireless transmitter parameters change dynamically \cite{dang2020should,zeng2024tutorial}. In 6G networks, rapid shifts in user distribution, environmental conditions, and personalized service demands create significant temporal-spatial variations in service requirements \cite{shen2023toward}. This necessitates the ability for network managers to rapidly adapt service strategies in real-time, in order to maintain service quality and efficiency \cite{zeng2021toward}. Traditional pre-computed RMs and CKMs, however, struggle to offer timely updates or support on-demand services in such dynamic environments, as they are often limited by their inability to respond quickly to evolving conditions \cite{zeng2024tutorial}. This research highlights the critical need for rapid RM inference, proposing an innovative solution where RMs can be quickly reconstructed following environmental or base station (BS) location changes, leveraging pre-calculated RMs or their intermediate variables. This capability aligns with the requirements of 6G networks, offering the agility and dynamism necessary to meet the challenges of next-generation wireless systems.

\begin{figure*}[t]
\captionsetup{font={small}, skip=16pt}
\centering
    \subfigure[The similarity of diffusion midpoint variables.]
    {
     \centering
           \includegraphics[width=0.95\linewidth]{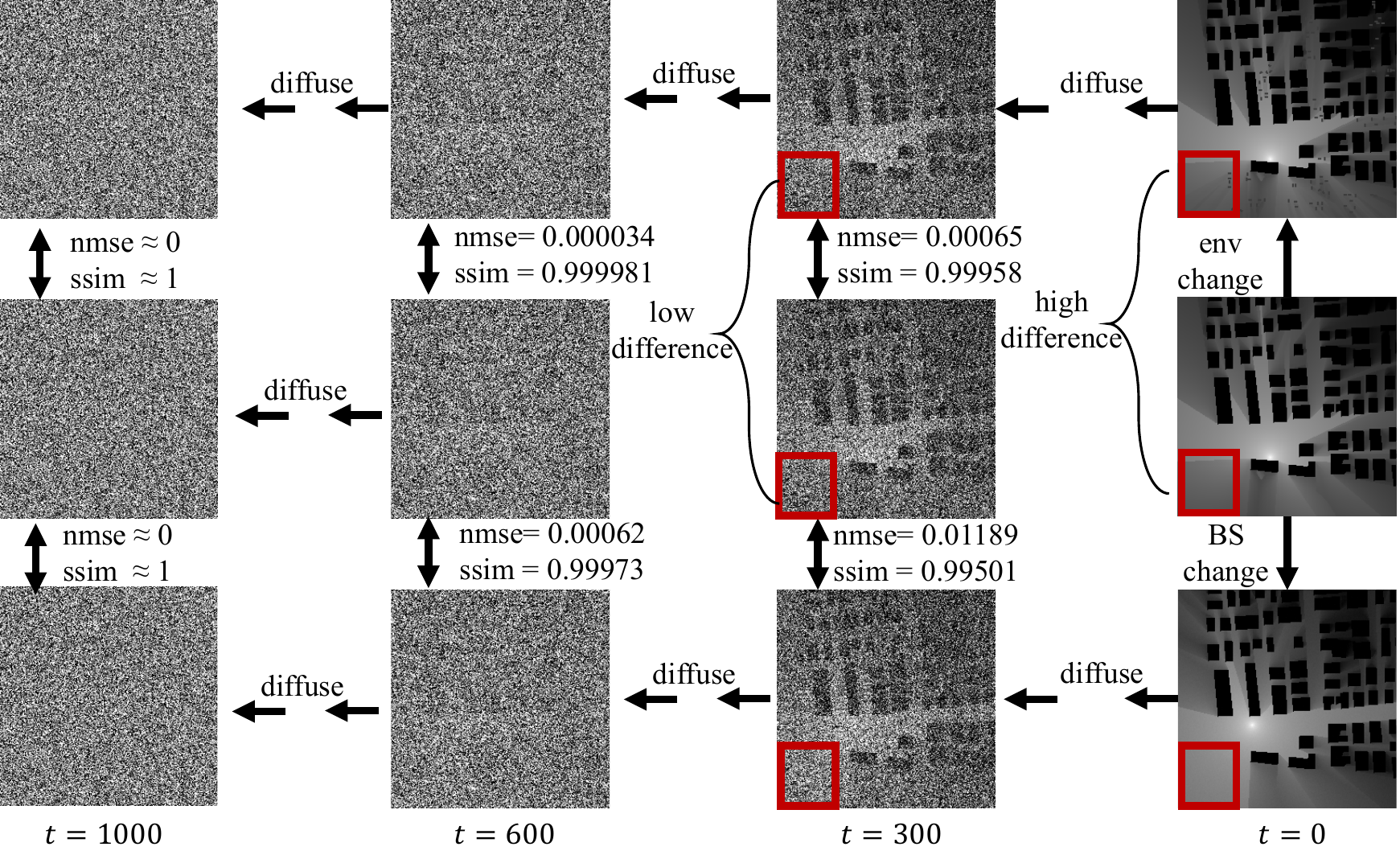}\label{diffusion-similar}
    }
    \subfigure[The metric similarity of diffusion midpoint variables between similar environments.]
    {
     \centering
           \includegraphics[width=0.475\linewidth]{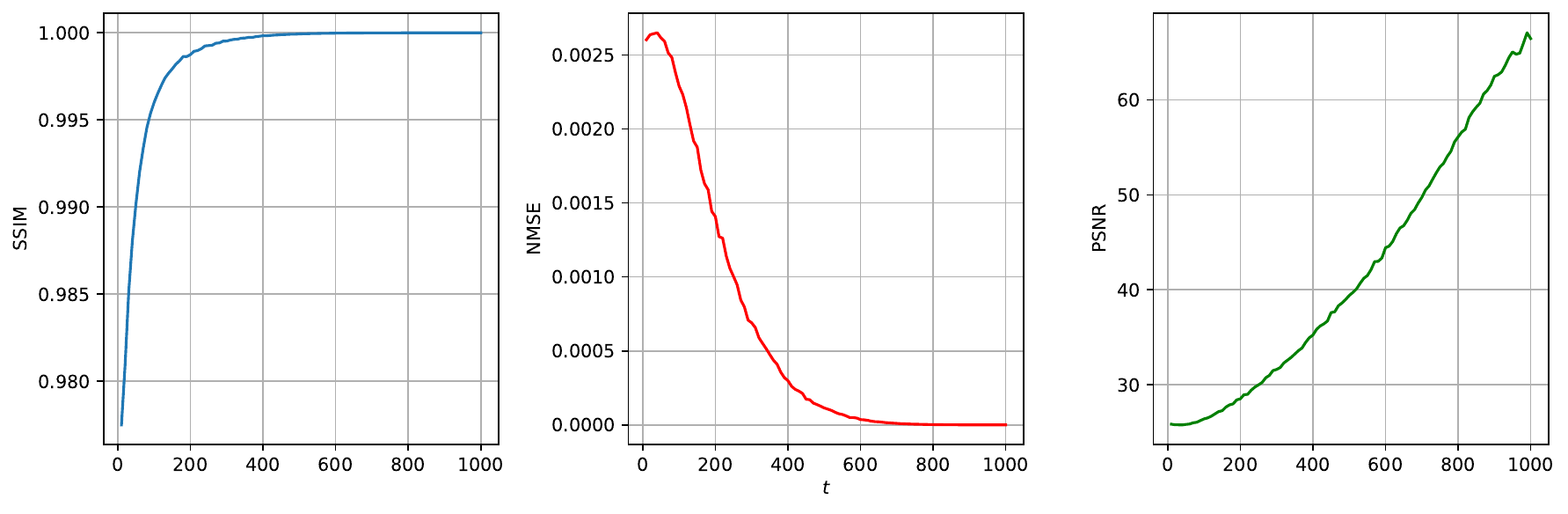}\label{fig-metric-env}
    }
    \subfigure[The metric similarity of diffusion midpoint variables between different BS locations.]
    {
     \centering
           \includegraphics[width=0.475\linewidth]{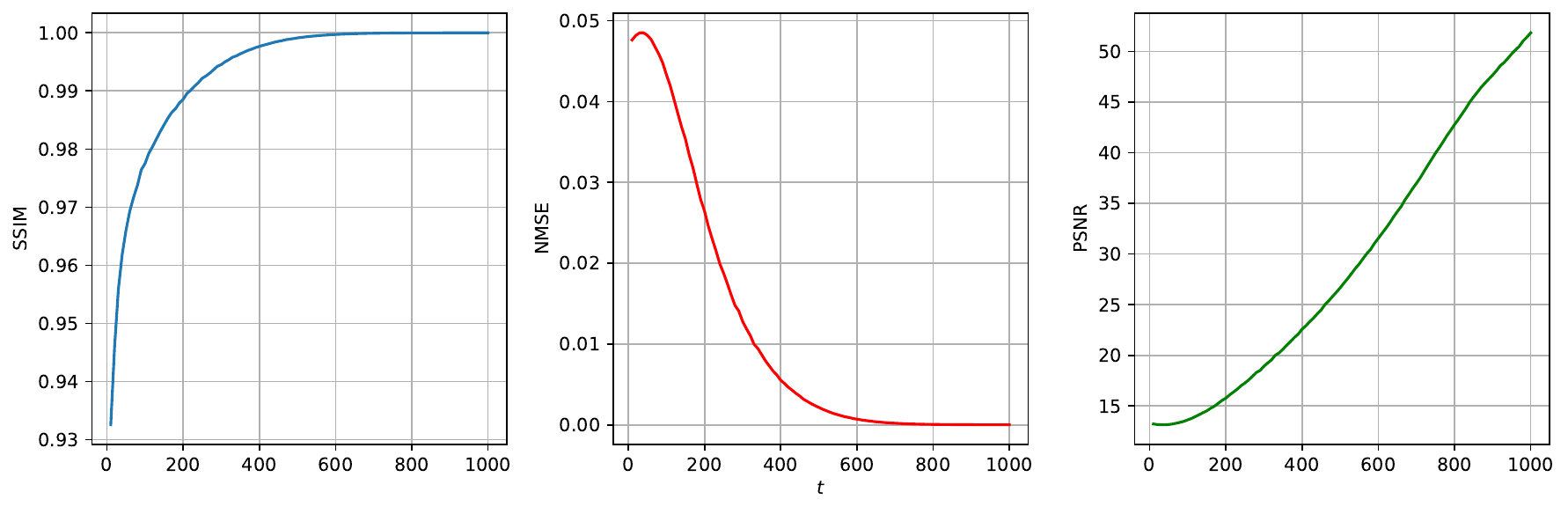}\label{fig-metric-bs}
    }
\caption{\added{The illustration of motivation for latent denoise reuse. (a) Visualizes the convergence of diffusion trajectories for semantically similar RMs. (b) and (c) quantitatively measure the similarity of latent variables over the diffusion process. The x-axis represents the diffusion timestep $t$, and the y-axis represents the NMSE between the latent variables of two different scenarios.}}
\label{fig_motivation}
\end{figure*}

Existing technologies can be divided into two paradigms: physical-driven \cite{deschamps1972ray,deschamps1972ray,irt} and data-driven methods \cite{levie2021radiounet,zhang2023rme,wang2024radiodiff}. However, both of them face fundamental limitations in dynamic scenarios. Physical-driven methods, such as ray tracing, simulate the propagation of electromagnetic waves by solving Maxwell's equations \cite{zhou2017electromagnetic}. Although they can achieve RM modeling with centimeter-level accuracy, their computational complexity is exponentially related to the size of the scene. Building a 100-meter resolution RM often requires tens of minutes of server-level computing power, and any slight environmental changes, such as vehicle movement, can cause global changes in the path of electromagnetic waves, forcing a complete recalculation \cite{oh2004mimo}. The rigid computing architecture of such methods obviously cannot adapt to the second-level RM update requirements in 6G scenarios. Data-driven methods attempt to break through the efficiency bottleneck by learning environmental feature mappings through neural networks. However, traditional discriminative models are good at regressing channel parameters from local features, they have difficulty in generating spatially coherent global RMs \cite{levie2021radiounet,li2022radionet}. Although generative adversarial networks (GANs) have the ability to generate data, their reliability in actual deployment is insufficient due to mode collapse and training instability \cite{zhang2023rme}. In recent years, diffusion models (DMs) have made significant progress in the task of RM construction with their progressive generation mechanism. The accuracy of DM-based methods can rival that of ray tracing \cite{wang2024radiodiff,wang2025radiodiff}. However, the iterative denoising process of DM requires thousands of neural network inferences, resulting in a generation delay of several seconds for a single RM, which still makes it difficult to support the real-time requirements of high-dynamic scenes \cite{LDM}. Moreover, the ``zero memory" generation mode of traditional DM completely ignores the temporal and spatial correlation in the continuous evolution of scenes. For example, when a drone moves along a trajectory, DM needs to perform a complete denoising process from scratch for each new location, while the similarity of the propagation laws implicitly existing between RMs of adjacent locations is not effectively utilized. This redundant calculation is not only inefficient, but also likely to introduce inter-frame jitter due to random noise initialization, which can destroy temporal consistency - this is particularly fatal for applications that require continuous RM sequences.

To address the limitations of existing RM construction methods in dynamic environments, our study reveals a key empirical finding, as shown in Fig.~\ref{fig-cluster}, the intermediate latent variables in the diffusion process exhibit strong similarity across scenarios with comparable environmental characteristics. For instance, when base station locations are slightly adjusted within the same building layout, the diffusion trajectories show highly consistent latent representations in the middle stages of denoising, even though the final RMs differ significantly. This observation indicates that the intermediate states primarily encode stable, scene-invariant features such as architectural structure and material properties, while later stages are responsible for refining scene-specific details like antenna radiation patterns and dynamic obstacles. This insight leads us to propose RadioDiff-Flux, a two-stage implicit diffusion framework designed to reuse the intermediate states, referred to as midpoints, within the generative process. By decoupling the modeling of static environmental features from the refinement of dynamic or transmitter-specific elements, our approach significantly enhances inference efficiency while preserving spatial and temporal consistency. The resulting framework enables scalable and low-latency RM generation, particularly well-suited for dynamic 6G scenarios. The main contributions of this paper are summarized as follows.
\begin{enumerate}
    \item We conduct a detailed analysis of the latent diffusion process and uncover that RMs generated from different base station positions or dynamic variations within the same static environment share highly similar diffusion midpoints. These midpoints capture stable environmental semantics, enabling their reuse to significantly reduce redundant inference in dynamic scenarios.
    \item To support this observation, we provide a theoretical analysis based on KL divergence, showing that RMs with similar structures exhibit closely aligned denoising trajectories. This offers a rigorous foundation for the feasibility and effectiveness of midpoint reuse.
    \item Leveraging these insights, we propose two implementations: vanilla midpoint reuse, which enables zero-cost adaptation using cached midpoints from a pre-trained model; and RadioDiff-Flux, a two-stage framework that decouples static and dynamic inference for improved accuracy and efficiency. Both designs facilitate rapid RM updates while preserving generative fidelity.
    \item Extensive experimental results demonstrate the effectiveness of our approach. In dynamic scenarios, the proposed midpoint reuse strategies achieve over 50× acceleration in inference speed, with less than 0.12\% degradation in RM construction accuracy, significantly advancing the practicality of diffusion-based RM generation for real-time and mobility-aware wireless systems.
\end{enumerate}

\section{Preliminary}
\begin{figure*}
    \centering
    \includegraphics[width=0.95\linewidth]{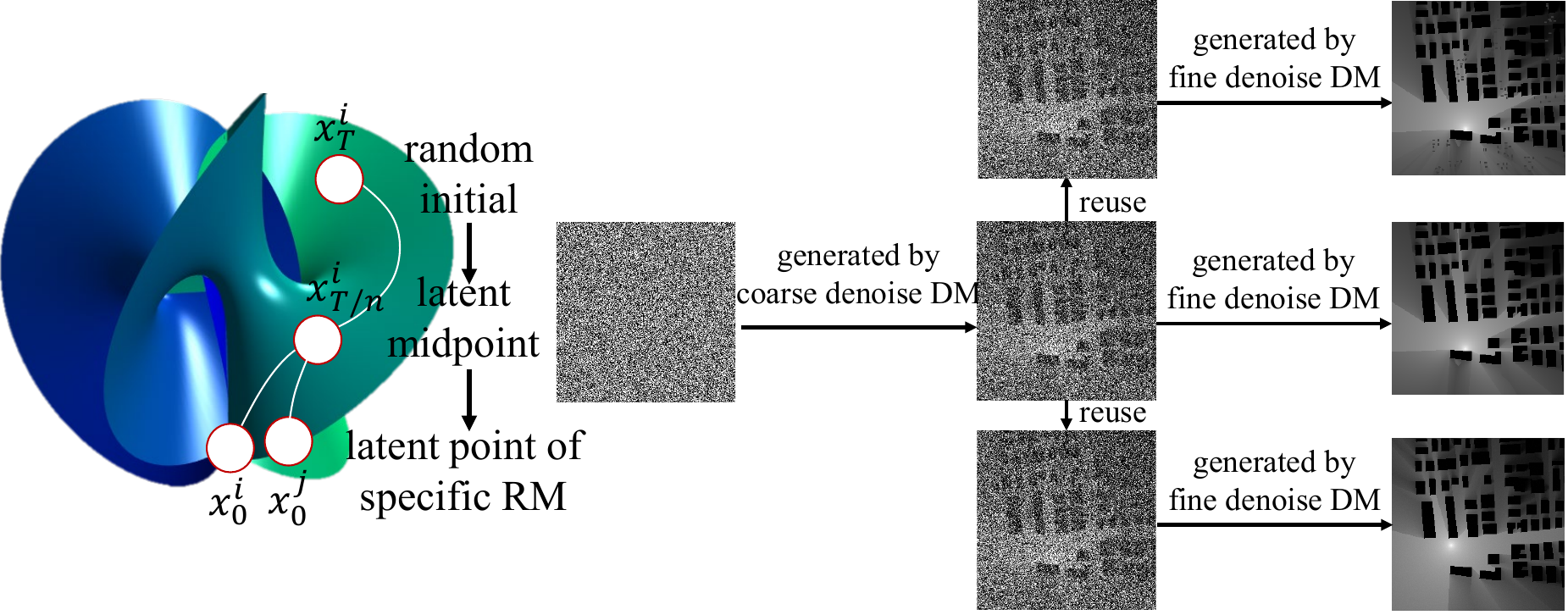}
    \caption{The illustration of latent midpoint reuse for the RM generation framework.}
    \label{fig-reuse}
\end{figure*}

DMs particularly denoising diffusion probabilistic models (DDPM) \cite{ho2020denoising,song2020denoising}, have emerged as a powerful class of generative models for various data synthesis tasks, including image generation \cite{avrahami2022blended}, denoising \cite{LDM}, and even in the context of wireless channel estimation \cite{wang2024radiodiff,11083758}. These models operate by progressively adding noise to data in a forward process and then learning to reverse the noise in a reverse denoising process. The main advantage of diffusion models lies in their ability to produce high-quality data through a gradual refinement process, making them well-suited for tasks requiring precise generation of complex data distributions, such as radio map construction \cite{jia2025rmdm, wang2025radiodiffinverse}.

The forward process in DDPM involves a Markov chain of length T, where at each step, Gaussian noise is progressively added to the original data. The transition from the clean data $x_0$ to the noisy data $x_T$ is defined by a series of conditional distributions, where at each step $t$, the data is perturbed by adding noise with variance controlled by a schedule. Formally, the forward process is modeled as follows \cite{ho2020denoising}.
\begin{align}
q(x_t | x_{t-1}) = \mathcal{N}(x_t; \sqrt{1 - \beta_t} \cdot x_{t-1}, \beta_t \cdot I),
\end{align}
where $\beta_t$ is a variance schedule that controls the amount of noise added at each step, $\mathcal{N}(\cdot)$ represents a Gaussian distribution, and $I$ is the identity matrix. The variable $x_t$ denotes the noisy version of the data at time step $t$, and the process continues until $t = T$, where the data $x_T$ is essentially pure noise. To simplify, the total forward process can be represented as follows.
\begin{align}
q(x_T | x_0) = \mathcal{N}(x_T; \sqrt{\bar{\alpha}_T} \cdot x_0, (1 - \bar{\alpha}_T) \cdot I),
\end{align}
where $\alpha_t = 1 - \beta_t$, and $\bar{\alpha}_T = \prod_{t=1}^{T} \alpha_t$. This cumulative process effectively maps the original data $x_0$ to a noisy sample $x_T$.

Once the data has been diffused to noise, the reverse process is learned, aiming to recover the original data from the noisy version. The reverse process is essentially the denoising step, where each noisy image $x_t$ is mapped back to the cleaner version $x_{t-1}$. The reverse dynamics are governed by the distribution as follows.
\begin{align}
p_\theta(x_{t-1} | x_t) = \mathcal{N}(x_{t-1}; \mu_\theta(x_t, t), \sigma_t^2 I),
\end{align}
where $\mu_\theta(x_t, t)$ is the mean of the distribution, predicted by a neural network, and $\sigma_t^2$ is the variance at step $t$. The model is trained to learn the denoising function $\mu_\theta(x_t, t)$ through the following objective.
\begin{align}
\mathcal{L} = \mathbb{E}_q \left[ \| \epsilon - \epsilon_\theta(x_t, t) \|^2 \right],
\end{align}
where $\epsilon$ is the noise added in the forward process, and $\epsilon_\theta(x_t, t)$ is the model’s predicted noise. The network learns to predict the noise that was added at each step of the diffusion, and the loss is minimized by comparing the predicted noise to the true noise. By reversing this process iteratively, starting from the noise $x_T$, the model progressively recovers the original data, $x_0$.

According to \cite{song2020score}, the diffusion process of latent variables in the generative model can be equivalently expressed by a stochastic differential equation (SDE) as follows \cite{huang2024decoupled}.
\begin{align}
    \mathrm{d} \bm{z}_t &= f_t \bm{z}_t \,\mathrm{d}t + g_t \,\mathrm{d} \bm{\epsilon}_t,\label{ddm-sde}\\
    f_t &= \frac{\mathrm{d} \log \gamma_t}{\mathrm{d} t}, \\
    g_t^2 &= \frac{\mathrm{d} \delta_t^2}{\mathrm{d} t} - 2 f_t \delta_t^2,
\end{align}
where $\bm{z}_t$ is the noisy latent representation at time $t$, $\bm{\epsilon}_t$ is standard Brownian noise, and $f_t$ and $g_t$ denote the drift and diffusion coefficients, respectively. The reverse process, which recovers $\bm{z}_0$ from $\bm{z}_t$, follows:
\begin{align}
    \mathrm{d} \bm{z}_t = \left[f_t \bm{z}_t - g_t^2 \nabla_{\bm{x}} \log q(\bm{z}_t)\right] \mathrm{d}t + g_t \mathrm{d} \overline{\bm{\epsilon}}_t,\label{ddm-reverse}
\end{align}
where $\overline{\bm{\epsilon}}_t$ is a Gaussian noise term from the time-reversed diffusion.

To enhance interpretability and modularity, we adopt a decoupled diffusion formulation that separates the denoising process into an additive structure:
\begin{align}
    \bm{z}_t &= \bm{z}_0 + \int_0^t \bm{f}_t \,\mathrm{d}t + \int_0^t \mathrm{d} \bm{\epsilon}_t, \label{decoupled-noise} \\
    \bm{z}_0 &+ \int_0^t \bm{f}_t \,\mathrm{d}t = \bm{0}, \label{decoupled-constraint}
\end{align}
where the first integral describes deterministic signal decay, and the second represents accumulated noise. Assuming the diffusion process is isotropic, the conditional distribution of $\bm{z}_t$ given $\bm{z}_0$ simplifies to:
\begin{align}
    q(\bm{z}_t|\bm{z}_0) = \mathcal{N}\left(\bm{z}_0 + \int_0^t \bm{f}_t \,\mathrm{d}t,\, t\bm{I}\right).\label{ddm-forward}
\end{align}
From this, we derive the reverse sampling distribution over a discrete step size $\Delta t$, which is essential for practical inference:
\begin{align}
q\left(\bm{z}_{t-\Delta t} \mid \bm{z}_t, \bm{z}_0\right)  &=\mathcal{N}\left(\bm{z}_{t} +\int_t^{t-\Delta t} \bm{f}_t \mathrm{~d} t\right. \notag\\
& \left.\qquad\qquad-\frac{\Delta t}{\sqrt{t}} \bm{\epsilon}, \frac{\Delta t(t-\Delta t)}{t} \bm{I}\right).\label{ddm-reverse}
\end{align}
The above DDM architecture is also used by the SOTA NN-based RM construction method in \cite{wang2024radiodiff}.

\section{System Model and Problem Formulation}
In this work, we consider a RM construction scenario over a discretized two-dimensional spatial region, represented as an $N \times N$ uniform grid. Each cell in the grid is assumed to be sufficiently small such that the pathloss within a cell remains approximately invariant. Consequently, the RM is defined as a matrix $\bm{P} \in \mathbb{R}^{N \times N}$, where each element $P(i,j)$ denotes the pathloss at location $(i,j)$. A single base station (BS) equipped with a dipole antenna is deployed as the sole radiation source in the environment. Its position is denoted by $r = \langle d_x, d_y, d_z \rangle$, where $(d_x, d_y)$ is the horizontal location and $d_z$ is the BS height. The environment contains static and dynamic obstacles, described by matrices $\bm{H}_s$ and $\bm{H}_d$ respectively. Static obstacles (e.g., buildings) are modeled as perfect electromagnetic (EM) shields, resulting in infinite pathloss ($P(i,j)=\infty$) in their interiors. In contrast, dynamic obstacles (e.g., vehicles) cause partial attenuation and scattering without fully blocking EM propagation. The entries $H_s(i,j) = 0$ and $H_d(i,j) = 0$ indicate the absence of static and dynamic obstacles at $(i,j)$, respectively. The goal is to learn a neural network $\bm{\mu}_\theta(\cdot)$ parameterized by $\bm{\theta}$ to predict the pathloss distribution $\hat{\bm{P}} = \bm{\mu}_\theta(\bm{H}_s, \bm{H}_d, r)$ that approximates the ground truth $\bm{P}$. The construction error is measured by a loss function $\mathcal{L}(\hat{\bm{P}}, \bm{P})$, typically the mean squared error (MSE).

However, in time-sensitive applications such as intelligent vehicular networks or drone-based coverage optimization, construction delay becomes a critical performance metric alongside accuracy. This is particularly relevant in the context of diffusion-based generative models, which are state-of-the-art in sampling-free RM construction due to their superior ability to model high-frequency textures and multi-modal uncertainty. Nevertheless, diffusion models involve iterative denoising over $T$ time steps, each requiring a full neural network evaluation, resulting in considerable computational delay. Let $\mathcal{C}_{\text{conv}}$ and $\mathcal{C}_{\text{attn}}$ denote the computational complexity per forward pass of a convolutional layer and an attention layer, respectively. These complexities scale as follows.
\begin{align}
\mathcal{C}_{\text{conv}} = \mathcal{O}(K^2 C_{\text{in}} C_{\text{out}} H W),
\mathcal{C}_{\text{attn}} = \mathcal{O}(H W d^2 + H^2 W^2 d),\label{eq-cnn-complexity}
\end{align}
where $K$ is the kernel size, $C_{\text{in}}$, $C_{\text{out}}$ are input/output channels, $H \times W$ is the spatial size, and $d$ is the feature dimension. These operations dominate the runtime of each neural forward pass.

For a denoising diffusion model, let $\mathcal{C}_{\text{net}}$ represent the complexity of a single neural network evaluation. The total computational complexity of the diffusion process is as follows.
\begin{align}
\mathcal{C}_{\text{DM}} = T \cdot \mathcal{C}_{\text{net}},
\end{align}
where $T$ is the number of denoising steps, which is typically from 500 to 1000. Due to the scaling law in deep learning, reducing $\mathcal{C}_{\text{net}}$ by shrinking the model size typically degrades performance, motivating the need to minimize $\mathcal{C}_{\text{DM}}$ by reducing $T$ or reusing partial computations.

Therefore, we propose a new formulation of the RM construction task as a bi-objective optimization problem that jointly minimizes both the construction error and the computational delay. Specifically, let $\mathcal{T}(\bm{\theta})$ denote the expected time to construct the RM using parameters $\bm{\theta}$, which can be approximated as $\mathcal{T}(\bm{\theta}) = T \cdot \tau(\bm{\theta})$, where $\tau(\bm{\theta})$ is the time for a single forward pass. We formulate the RM construction task as follows.
\begin{problem}\label{p1}
\begin{align}
    &\min_{\bm{\theta}, T} && \mathcal{L}(\hat{\bm{P}}, \bm{P}) + \lambda \cdot \mathcal{T}(\bm{\theta}) \label{obj-delay}\\
    &\text{s.t.} && \hat{\bm{P}} = \bm{\mu}_{\bm{\theta}}(\bm{H}_s, \bm{H}_d, r), \tag{\ref{obj-delay}a}
\end{align}
\end{problem}
\noindent where $\lambda > 0$ is a weighting coefficient that balances accuracy and inference speed. The constraint $T \leq T_{\max}$ ensures tractable inference latency. This formulation emphasizes the dual objective of accurate and efficient RM construction. It highlights the critical importance of optimizing not only the performance of the diffusion model but also the number of diffusion steps and architectural efficiency. It also provides a theoretical motivation for investigating the reuse of intermediate latent states or denoising acceleration strategies as explored in our proposed method. \added{It is important to note that this formulation primarily serves as a high-level motivation for our work, framing the inherent trade-off between construction accuracy and inference latency. The coefficient $\lambda$ represents the relative importance of speed versus accuracy. Instead of directly optimizing this objective via $\lambda$, our proposed framework, RadioDiff-Flux, addresses this trade-off architecturally by reducing the number of effective inference steps. Our experiments then empirically evaluate this trade-off by varying the reuse ratio $R_{\text{reuse}}$.}

\added{Our current model considers a single BS for clarity in formulation and evaluation. However, the framework can be extended to multi-BS environments. A practical approach is to generate an individual RM for each BS and then combine them using signal superposition principles, such as selecting the strongest signal at each location. In this context, the efficiency of RadioDiff-Flux becomes even more pronounced, as it can rapidly generate RMs for multiple BSs within the same static environment by reusing the pre-computed midpoint, significantly reducing the overall computation time compared to generating each map from scratch. A more integrated approach, which we leave for future work, would involve architecturally modifying the model to accept multiple BS locations as a single conditional input to generate a composite RM directly.}

\section{DM Midepoint Reusing}
\subsection{Motivation and Theoretical Analysis}
Recent advances in neural network-based RM generation, particularly those employing latent diffusion models (LDMs), have achieved significant improvements in construction fidelity by denoising within a compressed latent space. Although the primary motivation for LDMs lies in reducing inference complexity, insights from semantic communication, especially deep joint source-channel coding (Deep JSCC), reveal a deeper implication: the encoder’s latent feature maps inherently capture high-level semantic representations of environmental characteristics, such as obstacle layout and structural topology \cite{bourtsoulatze2019deep,xie2021deep}. In the context of RM construction, this implies that RMs generated under varying BS positions, but within the same static environment, should share substantial semantic information. This is visually corroborated by Fig.~\ref{fig-cluster} and Fig.~\ref{fig-metric-bs}, which illustrate that RMs with nearby BS positions yield tightly clustered embeddings in latent space, evidencing their shared environmental semantics.

To empirically validate this observation, we conduct a controlled study illustrated in Fig.~\ref{fig-metric-bs} and Fig.~\ref{fig-metric-env}. We consider three types of scenarios: (1) constant environment with varying BS positions, (2) fixed BS location with dynamic obstacles (e.g., moving vehicles), and (3) a reference RM. By introducing identical Gaussian noise—according to the diffusion forward process—across all three cases, we evaluate the normalized mean square error (NMSE) between generated samples at different diffusion steps. The results show that in cases with only dynamic variations, the NMSE between samples becomes negligible after approximately $t = 600$, and its derivative declines sharply near $t = 400$. This supports the hypothesis that the denoising paths of semantically similar RMs converge significantly in later diffusion stages, indicating that intermediate noisy representations (i.e., \(\bm{z}_t\)) can be effectively reused across related scenarios, which forms the foundation for our proposed fast inference strategy.

To further support this insight theoretically, we analyze the similarity between intermediate latent states using the Kullback-Leibler (KL) divergence. The following result quantifies how the divergence between two latent vectors, under the same diffusion noise level $t$, decreases as their semantic similarity increases:

\begin{theorem}
Let $\bm{z}_i$ and $\bm{z}_j$ be two latent vectors extracted by a variational autoencoder (VAE) from RMs under similar environmental conditions. After applying $t$ steps of the forward diffusion process as defined in Eq.~\eqref{ddm-forward}, their resulting distributions are $p(x) = \mathcal{N}((1 - t)\bm{z}_i, tI)$ and $q(x) = \mathcal{N}((1 - t)\bm{z}_j, tI)$, respectively. Then, the KL divergence between them satisfies:
\[
D_{\mathrm{KL}}(p \| q) = \frac{1}{2} \frac{(1 - t)^2}{t} \|\bm{z}_i - \bm{z}_j\|^2.
\]
\end{theorem}

\begin{proof}
Let the means be $\bm{\mu}_i = (1 - t)\bm{z}_i$ and $\bm{\mu}_j = (1 - t)\bm{z}_j$. Since both distributions share the same covariance matrix $tI$, the KL divergence is:
\[
D_{\mathrm{KL}}(p \| q) = \frac{1}{2}(\bm{\mu}_j - \bm{\mu}_i)^T (tI)^{-1}(\bm{\mu}_j - \bm{\mu}_i).
\]
Substituting $\bm{\mu}_j - \bm{\mu}_i = (1 - t)(\bm{z}_j - \bm{z}_i)$ and $(tI)^{-1} = \frac{1}{t}I$, we obtain:
\[
D_{\mathrm{KL}}(p \| q) = \frac{1}{2} \frac{(1 - t)^2}{t} \| \bm{z}_j - \bm{z}_i \|^2.
\]
\end{proof}
\noindent This theoretical result not only justifies our empirical findings but also provides an upper bound on divergence that decays quadratically with increasing $t$. It implies that, at sufficiently high diffusion steps, semantically similar latent vectors become indistinguishable in distribution. Thus, from both practical and theoretical standpoints, it is viable to reuse intermediate diffusion states when the underlying semantic content remains consistent. This forms the core innovation of our method, enabling accelerated RM construction by bypassing redundant denoising operations, without compromising estimation quality.

\added{In practice, to decide whether a cached midpoint can be reused across environments, we measure environment similarity in the cross-attention space of the pretrained RadioDiff backbone. The conditioning encoder yields tokens \(\bm{H}_s\), \(\bm{H}_d\), and \(\bm{r}\). We concatenate them as \(\mathcal{C}=[\bm{H}_s,\bm{H}_d,\bm{r}]\), then obtain pre-attention projections \(\bm{K}=\mathcal{C}\bm{W}_K\) and \(\bm{V}=\mathcal{C}\bm{W}_V\) using the frozen key and value matrices \(\bm{W}_K\) and \(\bm{W}_V\) from RadioDiff \cite{wang2024radiodiff}. Given two environments \(A\) and \(B\) with the same tokenizer and token layout so that tokens are aligned, we define a normalized Frobenius distance
\begin{align}
    D_{\mathrm{env}}(A,B)=\sqrt{\lVert \bm{K}^A-\bm{K}^B\rVert_F^2+\lVert \bm{V}^A-\bm{V}^B\rVert_F^2}.
\end{align}
Reuse is triggered when \(D_{\mathrm{env}}\le \tau\), with \(\tau\) selected on a small validation set at the knee of the accuracy versus reuse curve so that reuse respects a predefined error budget. This criterion is resolution agnostic since it operates on tokens, requires no retraining because the attention block is frozen, and leverages the fact that RadioDiff has learned geometry- and visibility-aware embeddings. Empirical examples of \(D_{\mathrm{env}}\) and the calibration of \(\tau\) are reported in the experimental section.
}

\subsection{Dual-DM based Midpoint Reuse for RM Construction}
To address the growing demand for low-latency and high-fidelity RM construction in dynamic wireless environments, we propose a two-stage conditional latent diffusion framework that explicitly decouples static environmental semantics from dynamic variations and transmitter-specific attributes. This architectural design is grounded in the observation that diffusion trajectories of semantically similar scenes exhibit strong convergence in intermediate latent space, as demonstrated in Section III. The proposed method aims to minimize redundant computation in early diffusion steps by strategically reusing shared semantic structures across similar scenarios. In the first stage of our framework, a dedicated latent diffusion model is trained to model coarse environmental semantics. This model is conditioned exclusively on static environmental context, such as the layout of buildings and large-scale terrain, and generates an intermediate latent state referred to as the diffusion midpoint. This midpoint serves as a high-level semantic representation of the scene, abstracted away from transient elements and transmitter configurations. In the second stage, a separate conditional diffusion model, now additionally conditioned on the dynamic environment, such as moving vehicles, and BS location, performs the remaining denoising steps to reconstruct the final RM. Importantly, this second-stage model initiates the generation process from the precomputed midpoint, thereby bypassing the computationally intensive early stages of denoising. This two-stage formulation introduces two critical advantages. First, in mobility-driven applications, such as those involving AAVs or mobile BSs, where static environmental features remain unchanged, the midpoint can be cached and reused across different BS placements. This allows the system to quickly adapt to new BS coordinates with minimal overhead. Second, in scenarios with fixed BS deployments but evolving dynamic conditions, our framework enables efficient RM updates by isolating the denoising effort to dynamic perturbations alone, leveraging the precomputed static-conditioned latent features. Such capabilities are crucial for enabling real-time responsiveness in 6G systems characterized by dense user mobility and environment variability.

\added{In this context, the ``diffusion midpoint" does not refer to a fixed temporal halfway point (i.e., $t=T/2$), but rather to any intermediate latent state $\bm{z}_t$ along the denoising trajectory, determined by the reuse ratio $R_{\text{reuse}}$. The selection of this point is flexible, and as our experiments show, the model's performance is sensitive to this choice, creating a direct trade-off between inference speed and reconstruction fidelity.} From a computational standpoint, our design also brings substantial efficiency gains by \added{reducing the complexity of the condition embedding module commonly used in conditional diffusion models. This two-stage approach provides a guaranteed reduction in computational complexity. The feature extracting network for a figure form data is usually initialized by a CNN layer. According to \eqref{eq-cnn-complexity}, by reducing the input channel depth from three to one for static features at this critical first step, we achieve a significant and direct decrease in the total floating point operations (FLOPs) required.}  As a result, both the total inference latency and computational footprint are substantially lowered. In summary, the proposed two-stage conditional latent diffusion model leverages both semantic reusability and architectural decoupling to achieve fast, adaptive, and scalable RM generation. By aligning model design with the theoretical insights into latent similarity among semantically related scenes, our method delivers robust performance across a range of dynamic and mobility-aware scenarios, thereby advancing the practicality of generative RM construction for next-generation wireless networks.

\section{Experiments}
\subsection{Datasets and Evaluation Metrics}
Our evaluation utilizes the RadioMapSeer dataset \cite{yapar2023first}, stemming from the pathloss RM construction challenge. This dataset comprises 700 unique urban maps, each detailing geographic features like buildings(ranging from 50 to 150 per map). For training, we selected 500 maps, reserving the remaining 200 for testing, ensuring no spatial overlap between the two subsets. Every map includes 80 transmitter locations and their corresponding ground truth RMs. The map data originates from OpenStreetMap, covering various cities including Ankara, Berlin, Glasgow, Ljubljana, London, and Tel Aviv. Standard physical parameters across the dataset are consistent: transmitter and receiver heights are 1.5 meters, and building heights are 25 meters. Each map is rendered as a 256 × 256 pixel binary morphological image, representing a 1m resolution grid where '1' signifies a building area and `0' non-building space. Transmitter positions are provided numerically and marked in the morphological image by setting the corresponding pixel to '1'. Transmissions occur at 23 dBm power and 5.9 GHz carrier frequency. Ground truth RMs, crucial for training, are generated based on Maxwell's equations, modeling pathloss from electromagnetic ray reflection and diffraction. Specifically, the static RM (SRM) ground truth considers only the impact of fixed buildings. For dynamic RMs (DRM), the ground truth incorporates effects from both static buildings and randomly positioned vehicles along roads, as illustrated by Fig.~\ref{fig-example}.

\begin{figure}[t]
\captionsetup{font={small}, skip=16pt}
    \centering
    \vspace{-9pt}
    \subfigure[The illustration of SRM.]
    {
       \centering
       \includegraphics[width=0.405\columnwidth]{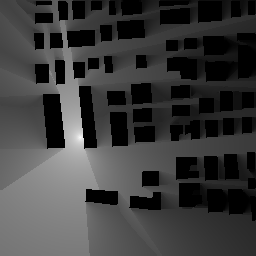}\label{srm}
    }
    \subfigure[The illustration of DRM.]
    {
       \centering
       \includegraphics[height=0.405\columnwidth]{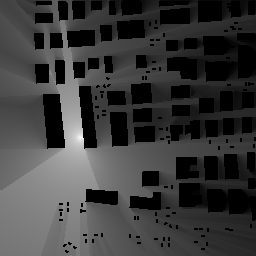}\label{drm}
    }
    \caption{Illustration of the RM. Pure black regions denote buildings or vehicles, signifying areas impassable to radio signals. The rest of the map is rendered as a grayscale image, with the grayscale level exhibiting a positive correlation to the pathloss value; brighter areas indicate higher pathloss.}
    \label{fig-example}
    \vspace{-3pt}
\end{figure}

To comprehensively evaluate the quality of constructed RMs, we employ several key metrics. We begin with standard error measures, Normalized Mean Squared Error (NMSE) and Root Mean Squared Error (RMSE), as in prior studies \cite{levie2021radiounet}. Recognizing that overall error metrics do not fully capture crucial structural details and integrity, we complement these with Structural Similarity Index Measurement (SSIM) and Peak Signal-to-Noise Ratio (PSNR). SSIM quantifies structural preservation, while PSNR assesses signal fidelity, particularly edge accuracy.
\subsubsection{MSE}
Mean Squared Error (MSE) measures the average squared difference between the ground truth and predicted RM pixel values. which can be calculated as $MSE=\frac1{NM}\Sigma_{m=0}^{M-1}\sum_{n=0}^{N-1}e(m,n)^2$, where $e(m,n)$ is the error at pixel $(m,n)$, and $M,N$ are image dimensions. NMSE scales MSE to the signal power, and RMSE provides an error measure in the same units as the data. which can be calculated as $NMSE=\frac{\Sigma_{m=1}^M\Sigma_{n=1}^N(I_b(m,n)-I(m,n))^2}{\Sigma_{m=1}^M\sum_{n=1}^NI^2(m,n)}$, and 
$\mathrm{RMSE}=\sqrt{MSE}$.
\subsubsection{SSIM}
SSIM evaluates image similarity considering luminance, contrast, and structural information, aligning well with the need to assess high-frequency details in RMs, which can be calculated as follows.
\begin{equation}l(x,y)=\frac{2\mu_X(x,y)\mu_Y(x,y)+C_1}{\mu_X^2(x,y)+\mu_Y^2(x,y)+C_1}\end{equation}
\begin{equation}c(x,y)=\frac{2\sigma_X(x,y)\sigma_Y(x,y)+C_2}{\sigma_X^2(x,y)+\sigma_Y^2(x,y)+C_2}\end{equation}
\begin{equation}s(x,y)=\frac{\sigma_{XY}(x,y)+C_3}{\sigma_X(x,y)\sigma_Y(x,y)+C_3}\end{equation}
where $x,y$ are the images, $\mu, \sigma^2,$ and $\sigma_{xy}$ are mean, variance, and covariance respectively. Constants $C_1=(K_1L)^2$, $C_2=(K_2L)^2$, and $C_3=C_2/2$ prevent division by zero, with $L$ being the data dynamic range. The final SSIM can be calculated as follows.
\begin{equation}SSIM(x,y)=\frac{(2\mu_x\mu_y+C_1)(\sigma_{xy}+C_2)}{(\mu_x^2+\mu_y^2+C_1)(\sigma_x^2+\sigma_y^2+C_2)}\end{equation}
\subsubsection{PSNR}
PSNR, expressed in dB, measures the ratio of maximum signal power to noise power, indicating reconstruction fidelity. A higher PSNR typically suggests better quality. For RMs, PSNR is particularly valuable for assessing the quality of reconstructed signal edges, which can be calculated as follows.
\begin{equation}PSNR=10\log_{10}\left(\frac{r^2}{MSE}\right)\end{equation}
where $r$ is the maximum possible pixel value in the image data.

\added{To quantify the caching footprint, we store midpoints in the latent space of a standard VAE. Each latent has size 64×64×4 in float32, which corresponds to 65,536 bytes, about 64 KB per cached midpoint. The cache size is independent of input resolution because it is maintained in latent space. In our deployments with fixed base-station layouts or mobile access points operating within a fixed region of interest, one cached midpoint that captures the static building layout is sufficient, about 64 KB in total. For city-scale service, we cap the cache at at most 100 midpoints, leading to about 6.25 MB in float32, and this can be reduced by half with float16 storage while preserving model behavior. For extracting features from the condition $\mathcal{C}$, we employ a Swin Transformer-B \cite{liu2021swin}, an architecture renowned for its powerful hierarchical feature representation. Our proposed two-stage method directly reduces the computational complexity of this feature extractor. The savings are realized at the transformer's initial patch embedding layer, which is implemented as a 3x3 convolution. By processing static conditions with a single-channel input instead of a three-channel one, we significantly cut down on the required operations. Specifically, for a typical input resolution of 256x256 and an embedding dimension of 128, this modification reduces the computational load by over 300 MFLOPs, making the entire pipeline more efficient before the feature maps are even passed to the subsequent self-attention blocks.}

\subsection{Experimental Methodology}
\label{sec:experimental_methodology}
The training pipeline of RadioDiff-Flux is fully aligned with that of RadioDiff \cite{wang2024radiodiff}, ensuring a consistent optimization framework. Specifically, the second-stage module, responsible for generating the complete RM conditioned on both environmental context and BS location, directly adopts the pre-trained weights from RadioDiff without any additional fine-tuning or retraining. Only the first-stage network, which infers the diffusion midpoint based solely on static environmental features, is retrained from scratch. This retraining process follows the same training strategy and loss configuration as used in RadioDiff, ensuring architectural and procedural consistency across both stages. To evaluate the efficacy of our proposed RadioDiff-Flux framework, particularly the midpoint reuse strategy, we first established baseline diffusion models and then conducted a series of experiments across diverse environmental change scenarios.

\subsubsection{Baseline Models}
\label{ssec:baseline_models}
To evaluate the effectiveness of our proposed method, we compare it against four representative baseline models that span both discriminative and generative paradigms for sampling-free RM construction as follows.
\begin{itemize}
    \item \textbf{RadioUNet} \cite{levie2021radiounet} serves as a foundational CNN-based benchmark for RM reconstruction. Leveraging the U-Net architecture, it is trained in a fully supervised manner to directly regress RMs from environmental features. Its simplicity and reliability have made it a standard reference in the field, particularly for evaluating discriminative methods.
    \item \textbf{UVM-Net} \cite{zheng2024u} builds on the same training protocol and input format as RadioUNet but replaces the convolutional backbone with a state space model, which enhances the architecture. Designed to handle long-range dependencies, SSMs project sequences into hidden state dynamics, allowing UVM-Net to better capture both localized textures and broader structural patterns. This modification makes it a compelling baseline for assessing how sequence modeling techniques improve spatial inference in complex environments.
    \item \textbf{RME-GAN} \cite{zhang2023rme} introduces a generative adversarial framework that originally combines environmental context with sparse pathloss measurements for conditional generation. For fair comparison under a sampling-free setup, we disable SPM input and use only environmental data. While RME-GAN showcases the potential of adversarial learning in wireless modeling, its reliance on sampled measurements in its canonical form limits its generalizability.
    \item \textbf{RadioDiff} \cite{wang2024radiodiff} currently represents the SOTA in RM generation, which formulates the task as a conditional generative process using a DM in latent space. By integrating a VAE for encoding and a UNet-based denoiser for reverse-time generation, RadioDiff captures both fine-grained spatial details and macro-scale pathloss structure. Its performance in terms of both accuracy and perceptual quality sets a strong baseline for advanced generative methods. 
    \item \textbf{Vanilla Midpoint Reuse (Ours)} implements a straightforward strategy for accelerating inference by directly reusing the cached midpoint of the diffusion trajectory. When either the base station location or the environmental configuration changes, the denoising process is initialized from this previously computed midpoint. The conditioning input of the diffusion model is then updated to reflect the new scenario, enabling reuse without requiring any fine-tuning or additional training. This approach leverages the pre-trained RadioDiff model in its entirety.
    \item \textbf{RadioDiff-Flux (Ours)} introduces a more structured two-stage framework. The first stage involves training a diffusion model conditioned solely on the static environment, enabling efficient generation of a semantically meaningful midpoint that captures large-scale spatial structures. The second stage then uses the pre-trained RadioDiff model to complete the denoising process, conditioned on both dynamic environmental features and base station location. This design allows for modular reuse of static information while maintaining high reconstruction fidelity in dynamically varying conditions.
\end{itemize}
\added{Notably, since RadioDiff-Flux directly inherits the architecture and pre-trained weights of RadioDiff for its core conditional generative stage, RadioDiff serves a dual role in our evaluation: it is not only the state-of-the-art benchmark but also the definitive ablation baseline for our framework operating without the midpoint reuse strategy.}

\subsubsection{Evaluation of Reuse Strategy under Environmental Changes}
\label{ssec:evaluation_scenarios}
With the trained models, we designed three experimental scenarios to assess the impact of varying the reuse ratio, $R_{\text{reuse}}$, on RM generation. The process begins by denoising from a Gaussian noise sample for a total of $T=100$ diffusion steps. Our reuse strategy involves changing the environmental conditioning information partway through this process, at a step determined by $R_{\text{reuse}}$.

\textbf{{Scenario 1: Changing Base Station Position.}}
\label{sssec:scenario_bs_change}
This scenario investigates reusing initial denoising steps when only the BS position changes within the same static environment. The process starts by running the static model, for a fraction of the total steps ($R_{\text{reuse}} \cdot T$) using an initial BS position. Then, for the remaining steps, the conditioning is switched to a new, target BS position. The final generated RM is then compared against the ground truth for this new position. This reuse approach aims to save computational resources compared to generating two RMs independently from scratch. We evaluated $R_{\text{reuse}}$ values of $[0.1, 0.4, 0.7, 0.9, 0.95, 0.98]$. The results are presented in Fig.~\ref{fig:results_scenario1_6cols}.

\begin{figure*}[htbp]
    \centering
    \captionsetup{font={small}}
    \begin{tabular}{@{}c@{\hspace{1mm}}c@{\hspace{1mm}}c@{\hspace{1mm}}c@{\hspace{1mm}}c@{\hspace{1mm}}c@{}}
        \includegraphics[width=0.15\linewidth]{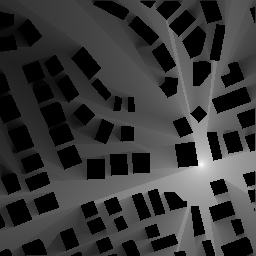} & 
        \includegraphics[width=0.15\linewidth]{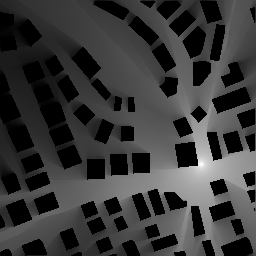} &
        \includegraphics[width=0.15\linewidth]{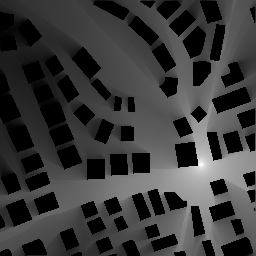} &
        \includegraphics[width=0.15\linewidth]{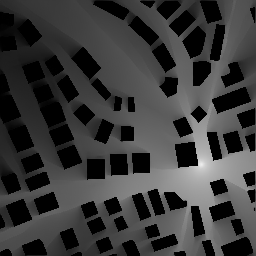} &
        \includegraphics[width=0.15\linewidth]{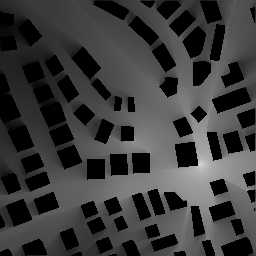} & 
        \includegraphics[width=0.15\linewidth]{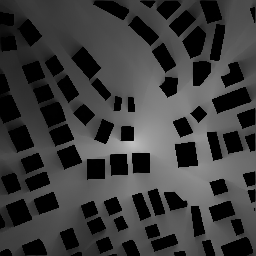} \\[2mm]
        
        \includegraphics[width=0.15\linewidth]{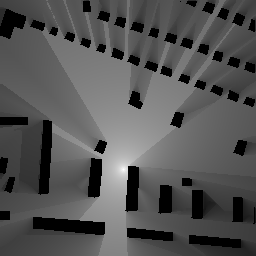} &
        \includegraphics[width=0.15\linewidth]{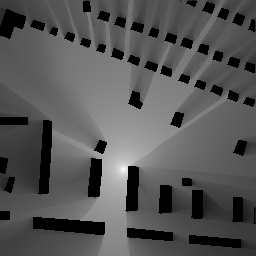} &
        \includegraphics[width=0.15\linewidth]{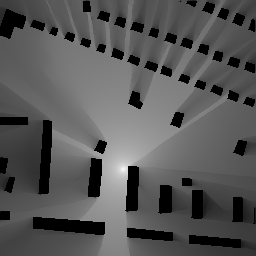} &
        \includegraphics[width=0.15\linewidth]{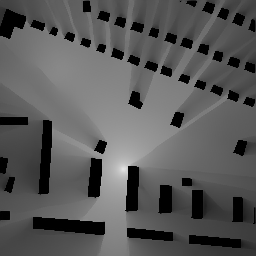} &
        \includegraphics[width=0.15\linewidth]{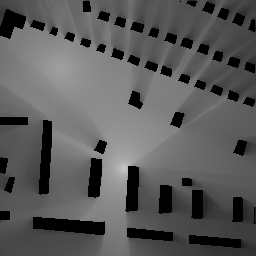} & 
        \includegraphics[width=0.15\linewidth]{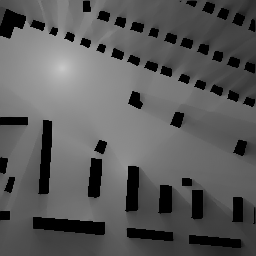} \\[2mm]

        \includegraphics[width=0.15\linewidth]{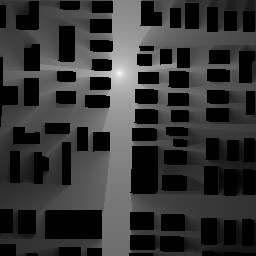} &
        \includegraphics[width=0.15\linewidth]{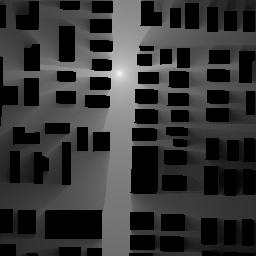} &
        \includegraphics[width=0.15\linewidth]{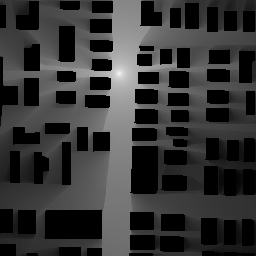} &
        \includegraphics[width=0.15\linewidth]{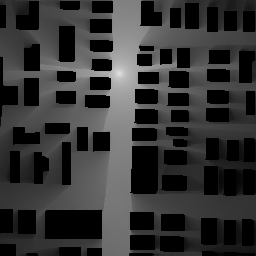} &
        \includegraphics[width=0.15\linewidth]{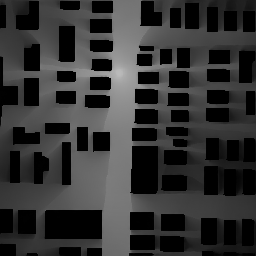} & 
        \includegraphics[width=0.15\linewidth]{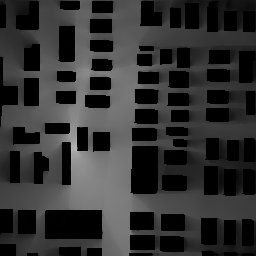} \\[2mm]
        
        \small Ground truth & \small $R_{\text{reuse}}=0.4$ & \small $R_{\text{reuse}}=0.7$ & \small $R_{\text{reuse}}=0.9$ & \small $R_{\text{reuse}}=0.95$ & \small $R_{\text{reuse}}=0.98$ \\
    \end{tabular}
    \caption{Visual comparison of RM generation under Scenario 1 (Base Station Position Change). Each row presents a distinct test case. The first column displays Ground Truth RMs. Subsequent columns illustrate generated RMs for varying trajectory reuse ratios ($R_{\text{reuse}}$).}
\label{fig:results_scenario1_6cols}
\end{figure*}

\textbf{{Scenario 2: Transitioning from Static to Dynamic Environment.}}
\label{sssec:scenario_dynamic_change}
This scenario simulates the introduction of dynamic elements, such as vehicles, into a previously static environment. The building layout and BS position remain fixed. The process begins by running the static model, for the initial fraction of steps. Then, we switch to the dynamic model, for the remaining steps, providing it with the new vehicle information. The final RM is evaluated against the ground truth that includes these vehicles. The results for different reuse ratios are shown in Fig.~\ref{fig:results_scenario2_6cols}.

\begin{figure*}[htbp]
    \centering
    \captionsetup{font={small}}
    \begin{tabular}{@{}c@{\hspace{1mm}}c@{\hspace{1mm}}c@{\hspace{1mm}}c@{\hspace{1mm}}c@{\hspace{1mm}}c@{}}
        \includegraphics[width=0.15\linewidth]{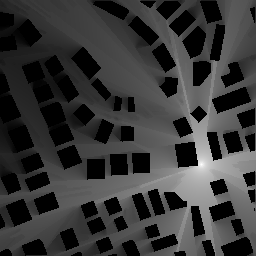} & 
        \includegraphics[width=0.15\linewidth]{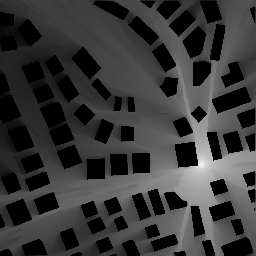} &
        \includegraphics[width=0.15\linewidth]{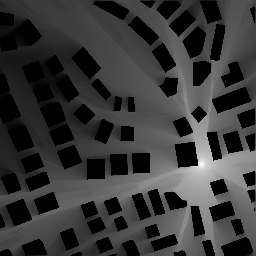} &
        \includegraphics[width=0.15\linewidth]{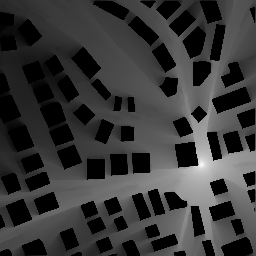} &
        \includegraphics[width=0.15\linewidth]{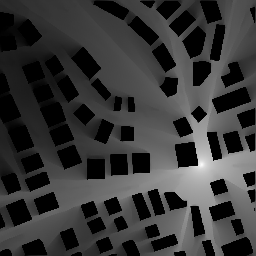} & 
        \includegraphics[width=0.15\linewidth]{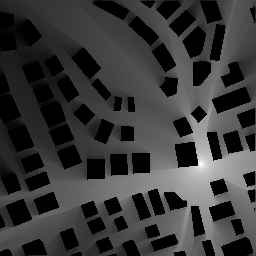} \\[2mm]
        
        \includegraphics[width=0.15\linewidth]{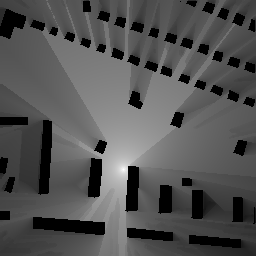} &
        \includegraphics[width=0.15\linewidth]{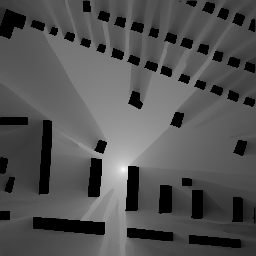} &
        \includegraphics[width=0.15\linewidth]{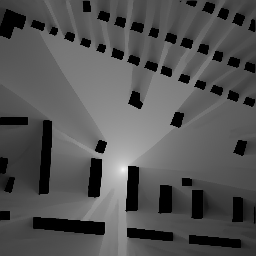} &
        \includegraphics[width=0.15\linewidth]{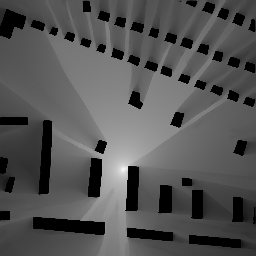} &
        \includegraphics[width=0.15\linewidth]{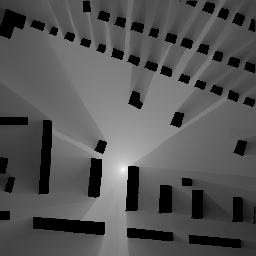} & 
        \includegraphics[width=0.15\linewidth]{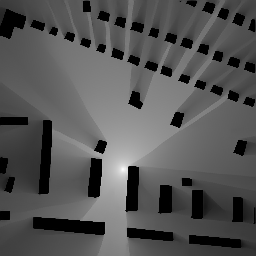} \\[2mm]

        \includegraphics[width=0.15\linewidth]{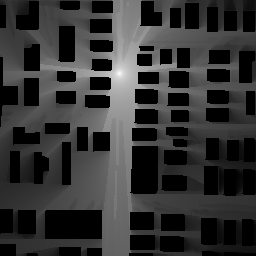} &
        \includegraphics[width=0.15\linewidth]{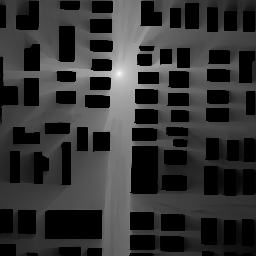} &
        \includegraphics[width=0.15\linewidth]{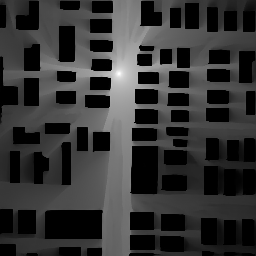} &
        \includegraphics[width=0.15\linewidth]{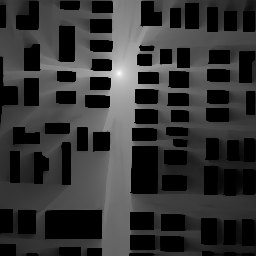} &
        \includegraphics[width=0.15\linewidth]{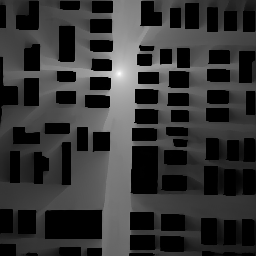} & 
        \includegraphics[width=0.15\linewidth]{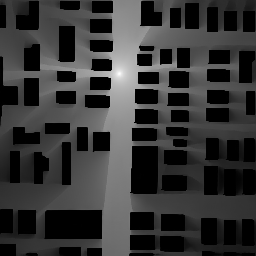} \\[2mm]
        
        \small Ground truth & \small $R_{\text{reuse}}=0.4$ & \small $R_{\text{reuse}}=0.7$ & \small $R_{\text{reuse}}=0.9$ & \small $R_{\text{reuse}}=0.95$ & \small $R_{\text{reuse}}=0.98$ \\
    \end{tabular}
    \caption{Visual comparison of RM generation under Scenario 2 (Transition from Static to Dynamic Environment). Each row presents a distinct test case. Ground Truth RMs are shown in the first column. Subsequent columns illustrate RMs generated using varying trajectory reuse ratios ($R_{\text{reuse}}$).}
\label{fig:results_scenario2_6cols}
\end{figure*}

\textbf{{Scenario 3: Directly Modifying the Static Environment.}}
\label{sssec:scenario_static_env_change}
This scenario assesses the reuse strategy when the static environment itself changes, involving alterations in both the building layout and the BS position. The static model is used throughout. For the initial fraction of steps, the model is conditioned on an initial static layout and BS position. For the remaining steps, it is conditioned on a new building layout and BS position. The final RM is evaluated against the ground truth for this new environment. The results are detailed in Fig.~\ref{fig:results_scenario3_6cols}.

\begin{figure*}[htbp]
    \centering
    \captionsetup{font={small}}
    \begin{tabular}{@{}c@{\hspace{1mm}}c@{\hspace{1mm}}c@{\hspace{1mm}}c@{\hspace{1mm}}c@{\hspace{1mm}}c@{}}
        \includegraphics[width=0.15\linewidth]{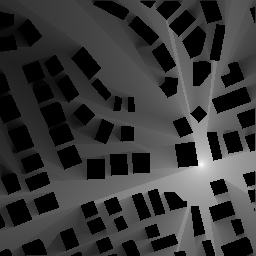} & 
        \includegraphics[width=0.15\linewidth]{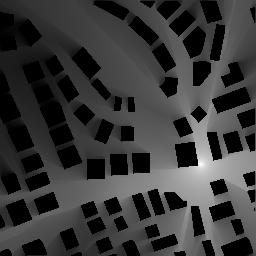} &
        \includegraphics[width=0.15\linewidth]{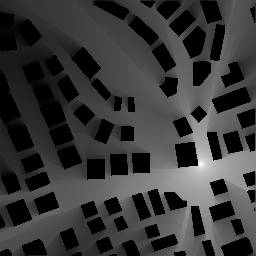} &
        \includegraphics[width=0.15\linewidth]{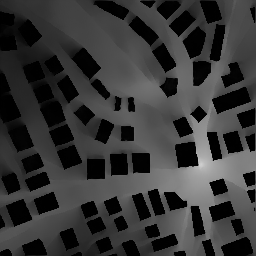} &
        \includegraphics[width=0.15\linewidth]{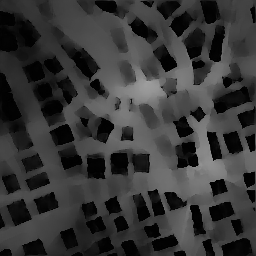} & 
        \includegraphics[width=0.15\linewidth]{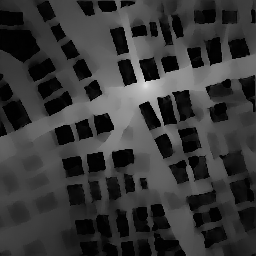} \\[2mm]
        
        \includegraphics[width=0.15\linewidth]{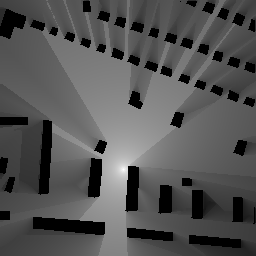} &
        \includegraphics[width=0.15\linewidth]{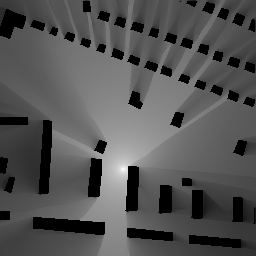} &
        \includegraphics[width=0.15\linewidth]{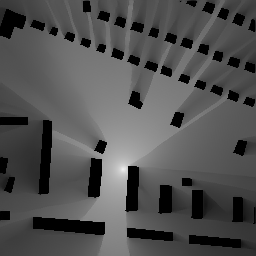} &
        \includegraphics[width=0.15\linewidth]{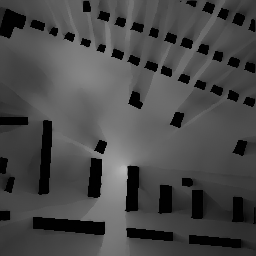} &
        \includegraphics[width=0.15\linewidth]{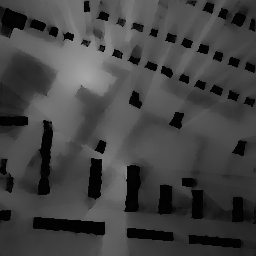} & 
        \includegraphics[width=0.15\linewidth]{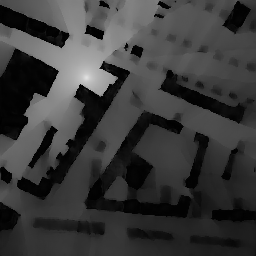} \\[2mm]

        \includegraphics[width=0.15\linewidth]{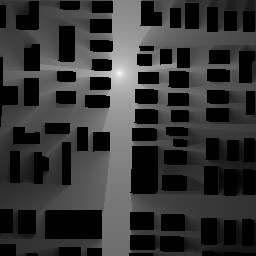} &
        \includegraphics[width=0.15\linewidth]{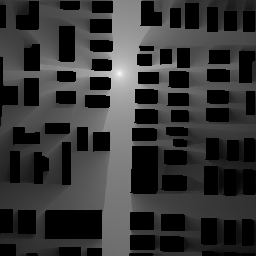} &
        \includegraphics[width=0.15\linewidth]{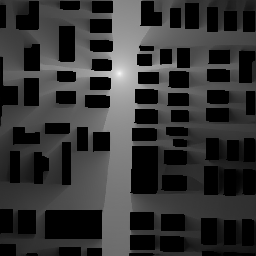} &
        \includegraphics[width=0.15\linewidth]{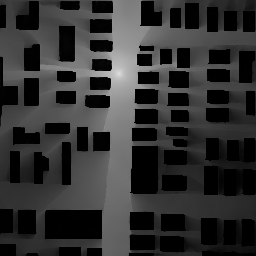} &
        \includegraphics[width=0.15\linewidth]{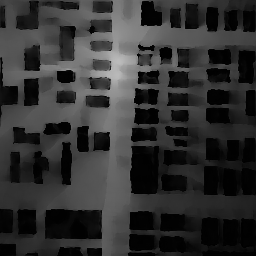} & 
        \includegraphics[width=0.15\linewidth]{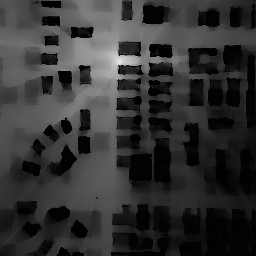} \\[2mm]
        
        \small Ground truth & \small $R_{\text{reuse}}=0.4$ & \small $R_{\text{reuse}}=0.7$ & \small $R_{\text{reuse}}=0.9$ & \small $R_{\text{reuse}}=0.95$ & \small $R_{\text{reuse}}=0.98$ \\
    \end{tabular}
    \caption{Visual comparison of RM generation under Scenario 3 (Static Environment Modification). Each row presents a distinct test case. The first column displays Ground Truth RMs. Subsequent columns illustrate RMs generated with varying trajectory reuse ratios ($R_{\text{reuse}}$).}
\label{fig:results_scenario3_6cols}
\end{figure*}

\subsection{Implementation Details}
\label{ssec:testing_protocol}
For robust statistical evaluation, 3000 independent trials were conducted using the test set for each configuration. In each trial, after setting initial and subsequent conditions, the final RM was generated and its performance evaluated against the ground truth of the second phase. The final results are reported as average metric scores over these trials. \added{All experiments were performed on a server equipped with an NVIDIA A40 GPU (48GB VRAM), running PyTorch version 2.2.0 with CUDA 11.8,} with the diffusion sampling process employing $T=100$ steps. Comprehensive test results are presented in Tables~\ref{tab:scenario1_results}, \ref{tab:scenario2_results}, and \ref{tab:scenario3_results}.

\added{Regarding the training cost, RadioDiff-Flux demonstrates significant computational efficiency due to its strategic reuse of pre-trained models. The entire second stage of our framework, which is responsible for the final radio map generation, is effectively training-free as it directly employs the identical architecture and pre-trained weights of the original RadioDiff model. Similarly, the VAE component in the first stage is also directly reused without modification. The only new training required is for the first-stage conditional diffusion model. Instead of training from scratch, we initialize this model with the weights from RadioDiff and fine-tune it for approximately 10 epochs. This fine-tuning process takes about 12 hours on our hardware, representing a dramatic reduction from the roughly 480 hours required to train the standard RadioDiff \cite{wang2024radiodiff} model from the ground up. This approach makes our method not only effective but also highly practical in terms of computational resources.}

\subsection{Result Analysis}
\label{sec:result_analysis}
This section provides a detailed analysis of the experimental results, evaluating the performance of our framework by examining the quantitative metrics and qualitative visual outcomes. The core of the analysis focuses on the impact of the reuse ratio, $R_{\text{reuse}}$, on RM construction accuracy and inference speed across the three defined scenarios.

\begin{table*}[t]
    \centering
    \caption{Quantitative performance for Scenario 1 (Changing Base Station Position). \added{The RadioDiff entry also serves as an ablation baseline, representing the performance of our architecture without the proposed midpoint reuse strategy.}}
    \label{tab:scenario1_results}
    \resizebox{0.6\linewidth}{!}{%
    \begin{tabular}{c | c | c c c c c}
    \toprule
    \textbf{Method} & \textbf{$R_{\text{reuse}}$} & \textbf{NMSE} & \textbf{RMSE} & \textbf{SSIM} & \textbf{PSNR (dB)} & \textbf{Time (ms)} \\
    \midrule
    RME-GAN & - & 0.01150 & 0.03030 & 0.93230 & 30.54000 & 42  \\
    UVM-Net & - & 0.00850 & 0.03040 & 0.93200 & 30.34000 & 95 \\
    RadioUnet & - & 0.00740 & 0.02440 & 0.95920 & 32.01000 & 60 \\
    \midrule
    RadioDiff & \added{0.00} & 0.00580 & 0.01990 & 0.96474 & 34.67248 & 600\\
    \midrule
    \multirow{11}{*}{Vanilla Midpoint Reuse (Ours)} & 0.10 & 0.00581 & 0.01991 & 0.96474 & 34.66587 & 532 \\
    & 0.20 & 0.00582 & 0.01995 & 0.96472 & 34.64317 & 477 \\
    & 0.30 & 0.00586 & 0.02001 & 0.96467 & 34.61296 & 416 \\
    & 0.40 & 0.00592 & 0.02013 & 0.96462 & 34.55543 & 356 \\
    & 0.50 & 0.00603 & 0.02034 & 0.96448 & 34.45523 & 301 \\
    & 0.60 & 0.00625 & 0.02075 & 0.96421 & 34.27014 & 236 \\
    & 0.70 & 0.00671 & 0.02159 & 0.96368 & 33.92364 & 173 \\
    & 0.80 & 0.00797 & 0.02369 & 0.96234 & 33.13318 & 120 \\
    & 0.90 & 0.01542 & 0.03297 & 0.95571 & 30.47665 & 63 \\
    & 0.95 & 0.04271 & 0.05474 & 0.93620 & 26.33720 & 31 \\
    & 0.98 & 0.13098 & 0.09766 & 0.88361 & 21.20604 & 12 \\
    \bottomrule
    \end{tabular}
    }
\end{table*}

\begin{table}[htbp]
\centering
\captionsetup{font={small}}
\caption{Quantitative performance of RadioDiff-Flux in Scenario 1 at high $R_{\text{reuse}}$ values.}
\label{tab:scenario1_latent_averaging_results}
\resizebox{0.8\linewidth}{!}{%
\begin{tabular}{@{}c | c c c c @{}}
\toprule
$R_{\text{reuse}}$ & NMSE      & RMSE      & SSIM      & PSNR (dB) \\ \midrule
0.98 & 0.02957 & 0.04567 & 0.94584 & 27.52674 \\
0.95 & 0.01292 & 0.02968 & 0.95849 & 31.28677 \\
0.90 & 0.00832 & 0.02381 & 0.96261 & 33.18820 \\
0.80 & 0.00655 & 0.02114 & 0.96439 & 34.23921 \\
0.70 & 0.00607 & 0.02035 & 0.96487 & 34.59251 \\ \bottomrule
\end{tabular}
}
\end{table}

\begin{table*}[t]
    \centering
    \caption{Quantitative performance for Scenario 2 (Transitioning from Static to Dynamic Environment). \added{The RadioDiff entry also serves as an ablation baseline, representing the performance of our architecture without the proposed midpoint reuse strategy.}}
    \label{tab:scenario2_results}
    \resizebox{0.6\linewidth}{!}{%
    \begin{tabular}{c | c | c c c c c}
    \toprule
    \textbf{Method} & \textbf{$R_{\text{reuse}}$} & \textbf{NMSE} & \textbf{RMSE} & \textbf{SSIM} & \textbf{PSNR (dB)} & \textbf{Time (ms)} \\
    \midrule
    RME-GAN & - & 0.01180 & 0.03070 & 0.92190 & 30.40000 & 42  \\
    UVM-Net & - & 0.00880 & 0.03010 & 0.93260 & 30.42000 & 95 \\
    RadioUNet & - & 0.00890 & 0.02580 & 0.94100 & 31.75000 & 60 \\
    \midrule
    RadioDiff & \added{0.00} & 0.00643 & 0.02239 & 0.95325 & 33.22775 & 600\\ \midrule
    \multirow{11}{*}{Vanilla Midpoint Reuse (Ours)}
    & 0.10 & 0.00643 & 0.02238 & 0.95317 & 33.23308 & 546 \\
    & 0.20 & 0.00642 & 0.02236 & 0.95313 & 33.23967 & 461 \\
    & 0.30 & 0.00640 & 0.02233 & 0.95312 & 33.24747 & 396 \\
    & 0.40 & 0.00640 & 0.02233 & 0.95307 & 33.24915 & 348 \\
    & 0.50 & 0.00639 & 0.02234 & 0.95299 & 33.23908 & 301 \\
    & 0.60 & 0.00642 & 0.02238 & 0.95294 & 33.22383 & 224 \\
    & 0.70 & 0.00646 & 0.02247 & 0.95293 & 33.18599 & 171 \\
    & 0.80 & 0.00659 & 0.02271 & 0.95278 & 33.09520 & 125 \\
    & 0.90 & 0.00688 & 0.02325 & 0.95247 & 32.88742 & 57 \\
    & 0.95 & 0.00732 & 0.02401 & 0.95170 & 32.61643 & 28 \\
    & 0.98 & 0.00776 & 0.02474 & 0.95094 & 32.37078 & 10 \\
    \bottomrule
    \end{tabular}
    }
\end{table*}
    
\begin{table*}[t]
    \centering
    \caption{Quantitative performance of Vanilla Midpoint Reuse for Scenario 3 (Directly Modifying the Static Environment).}
    \label{tab:scenario3_results}
    \resizebox{0.4\linewidth}{!}{%
    \begin{tabular}{c | c c c c c}
    \toprule
    \textbf{$R_{\text{reuse}}$} & \textbf{NMSE} & \textbf{RMSE} & \textbf{SSIM} & \textbf{PSNR (dB)} & \textbf{Time (ms)} \\
    \midrule
    0 & 0.00680 & 0.02092 & 0.96152 & 34.22417 & 600 \\
    0.10 & 0.00681 & 0.02093 & 0.96152 & 34.21530 & 537 \\
    0.20 & 0.00685 & 0.02099 & 0.96145 & 34.17983 & 472 \\
    0.30 & 0.00692 & 0.02112 & 0.96135 & 34.11101 & 412 \\
    0.40 & 0.00705 & 0.02138 & 0.96114 & 33.98104 & 359 \\
    0.50 & 0.00729 & 0.02184 & 0.96076 & 33.75646 & 296 \\
    0.60 & 0.00778 & 0.02273 & 0.95995 & 33.36347 & 239 \\
    0.70 & 0.00889 & 0.02458 & 0.95795 & 32.64405 & 175 \\
    0.80 & 0.01267 & 0.02986 & 0.94886 & 30.96165 & 117 \\
    0.90 & 0.04694 & 0.05926 & 0.87976 & 25.09397 & 58 \\
    0.95 & 0.19226 & 0.12455 & 0.66350 & 18.40915 & 28 \\
    0.98 & 0.58418 & 0.22358 & 0.37236 & 13.11929 & 11 \\
    \bottomrule
    \end{tabular}
    }
\end{table*}

\subsubsection{Scenario 1}
When only the BS position changes, the midpoint reuse strategy demonstrates remarkable efficacy at low to moderate reuse ratios. As shown in Table~\ref{tab:scenario1_results}, increasing $R_{\text{reuse}}$ up to 0.7 results in only a marginal increase in NMSE, accompanied by substantial gains in inference speed (e.g., a 3.47$\times$ speedup at $R_{\text{reuse}}=0.7$). Fig.~\ref{fig:results_scenario1_6cols} confirms that for $R_{\text{reuse}}$ up to 0.8, the generated RMs maintain high visual similarity to the ground truth. At very high reuse ratios ($R_{\text{reuse}} \geq 0.9$), accuracy degrades more noticeably, though speedups are significant. This indicates that an insufficient number of denoising steps remain for the model to adapt the latent representation to the new BS position. For this scenario, an $R_{\text{reuse}}$ between 0.7 and 0.8 offers an optimal balance between speed and accuracy.

To mitigate the 'blurred superposition' effect at high reuse ratios, we applied our RadioDiff-Flux method. As shown in Table~\ref{tab:scenario1_latent_averaging_results} and Fig.~\ref{fig:results_scenario1_latent_averaging}, this approach significantly improves performance. At $R_{\text{reuse}}=0.98$, NMSE is reduced from 0.13098 to 0.02957. This demonstrates that creating a generalized midpoint makes the model less susceptible to the influence of a single initial condition, offering a valuable refinement for scenarios requiring rapid updates with better fidelity.

\begin{figure}[htbp]
    \centering
    \captionsetup{font={small}}
    \begin{tabular}{@{}c@{\hspace{1mm}}c@{\hspace{1mm}}c@{}}
        \includegraphics[width=0.3\linewidth]{figure/scenario_resu/resu_s1/13_1_0.02_GT.png} & 
        \includegraphics[width=0.3\linewidth]{figure/scenario_resu/resu_s1/13_1_0.02.png} & 
        \includegraphics[width=0.3\linewidth]{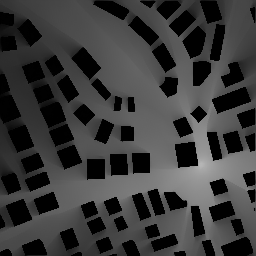} \\[0.5mm] 
        
        \includegraphics[width=0.3\linewidth]{figure/scenario_resu/resu_s1/418_0_0.02_GT.png} &
        \includegraphics[width=0.3\linewidth]{figure/scenario_resu/resu_s1/418_0_0.02.png} & 
        \includegraphics[width=0.3\linewidth]{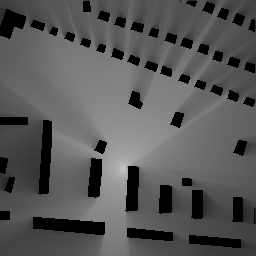} \\[0.5mm] 

        \includegraphics[width=0.3\linewidth]{figure/scenario_resu/resu_s1/520_1_0.02_GT.png} &
        \includegraphics[width=0.3\linewidth]{figure/scenario_resu/resu_s1/520_1_0.02.png} & 
        \includegraphics[width=0.3\linewidth]{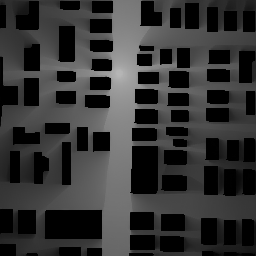} \\[2mm] 
        
        \small Ground Truth & \small Vanilla Midpoint Reuse & \small RadioDiff-Flux \\
    \end{tabular}
    \caption{Visual comparison for Scenario 1 at $R_{\text{reuse}}=0.98$: Ground Truth (left), original midpoint reuse (middle), and RadioDiff-Flux (right). Images in the rightmost column are placeholders.}
\label{fig:results_scenario1_latent_averaging}
\end{figure}

\subsubsection{Scenario 2}
This scenario reveals exceptional robustness, even at very high reuse ratios. The initial steps, using the static model, establish a strong foundation, which is then efficiently updated by the dynamic model. Quantitatively, as seen in Table~\ref{tab:scenario2_results}, accuracy remains remarkably high across all reuse ratios. Even at $R_{\text{reuse}}=0.98$, NMSE is a mere 0.00776, with a significant speedup of 58.07$\times$.

However, the introduction of vehicles adds high-frequency details. As observed in Fig.~\ref{fig:results_scenario2_6cols}, these dynamic details diminish with increasing $R_{\text{reuse}}$. While NMSE and RMSE signify excellent global accuracy, the SSIM metric, which is sensitive to structural information, better captures this loss of fine detail through its consistent, albeit small, decline. This scenario underscores the framework's capability for massive acceleration with minimal global error when adding dynamic elements, though very high reuse might trade off fine, dynamic features.

\subsubsection{Scenario 3}
When the static environment itself is modified, the impact of reuse is more critical. For $R_{\text{reuse}}$ up to 0.7, the generated RMs still largely reflect the target conditions, achieving a significant speedup (3.50$\times$) with reasonable accuracy (Table~\ref{tab:scenario3_results}). However, performance is less robust than in the other scenarios, as the change in underlying static features is more substantial.

\added{At high reuse ratios ($R_{\text{reuse}} \ge 0.8$), performance deteriorates sharply, as seen in Fig. 6. The generated RMs often retain strong, erroneous features from the initial environment. This highlights a key limitation of the midpoint reuse strategy: when the fundamental static layout changes significantly, the initial latent representation creates a strong ``inertial bias" that the model cannot overcome in the few remaining denoising steps. This clarifies the boundary of our method's effectiveness, indicating that for drastic environmental changes, a lower reuse ratio or a full regeneration from noise is necessary to ensure accuracy.}

\subsubsection{Overall Discussion and Concluding Remarks}
The experimental results consistently demonstrate a trade-off between the reuse ratio and RM accuracy, with sensitivity varying by the nature of the environmental change. Our framework's midpoint reuse is most effective when the semantic similarity is high between the initial and target conditions. Scenario 2 shows the highest robustness, achieving massive speedups with excellent global accuracy. Scenario 1 also performs well, especially with the RadioDiff-Flux refinement for high reuse ratios. Scenario 3 is the most challenging, where only moderate reuse offers a good balance. These findings validate that reusing intermediate diffusion states is a highly effective strategy for accelerating RM generation. This suggests adaptive $R_{\text{reuse}}$ strategies: high reuse for minor perturbations and conservative reuse for substantial reconfigurations. Our framework can accelerate RM generation by 3.5$\times$ to over 58$\times$ while maintaining high fidelity in many practical cases, showcasing its potential for near real-time RM updates crucial for dynamic wireless environments. \added{It is worth noting that while our experiments were conducted on a high-performance GPU, the significant relative speedup achieved is a key enabler for deployment on resource-constrained edge devices. The ability to reduce inference time by an order of magnitude makes near real-time RM adaptation feasible even on less powerful hardware, a crucial step toward practical implementation in mobile 6G systems.}

\added{These findings suggest the potential for an adaptive $R_{\text{reuse}}$ strategy in practical deployments: a high reuse ratio could be employed for minor perturbations, such as small BS movements or dynamic obstacle changes, while a more conservative ratio would be appropriate for substantial environmental reconfigurations. Developing a lightweight mechanism to quantify the magnitude of change between scenarios to automatically select an optimal $R_{\text{reuse}}$ remains a valuable direction for future work.}

\section{Conclusion}
In this paper, we have proposed RadioDiff-Flux, a novel framework for efficient RM construction by innovatively reusing trajectory midpoints within a generative denoising diffusion model. Theoretical analysis and experiments confirm that RadioDiff-Flux can substantially reduce inference latency, while maintaining high RM fidelity. Therefore, RadioDiff-Flux should offer a vital step towards real-time, high-accuracy RM generation, addressing a key bottleneck for adaptive wireless network management in 6G. For future work, we will explore adaptive midpoint reuse strategies based on environmental changes and enhancing temporal consistency for sequential RM generation.

\bibliography{ref}
\bibliographystyle{IEEEtran}

\ifCLASSOPTIONcaptionsoff
  \newpage
\fi

\end{document}